\documentclass[letterpaper, 10 pt, conference, anonymous]{ieeeconf} 
\IEEEoverridecommandlockouts    
\usepackage{amsmath, amssymb, amsfonts}
\usepackage{graphicx}
\usepackage{multicol}
\usepackage{soul}
\usepackage{float}
\usepackage[ruled,longend]{algorithm2e}
\usepackage[backend=biber, sorting=none]{biblatex}
\usepackage{mathtools}
\usepackage{stmaryrd}
\usepackage[draft]{changes}
\usepackage{subcaption}
\usepackage{booktabs}
\usepackage{censor}

\newtheorem{theorem}{Theorem}[section]

\newtheorem{proposition}{Proposition}

\definechangesauthor[name={Hemanth}, color=red]{HM}

\title{pdSTL: Probabilistic Differentiable Signal Temporal Logic for Stochastic Systems}
\author{Bennett Dogbey, IEEE Student Member and Hemanth Manjunatha, IEEE Member
\thanks{All Authors are with Dept. of Mechanical and Aerospace Engineering, Oklahoma State University, Stillwater, OK USA}}

\begin{document}
\maketitle
\thispagestyle{empty}
\pagestyle{empty}
\begin{abstract}
Autonomous robots operating in uncertain environments must satisfy complex temporal and safety specifications despite stochastic dynamics and sensing noise. While Signal Temporal Logic (STL) offers robustness measures for gradient-based optimization, existing extensions either lack differentiability or ignore belief-space uncertainty. We introduce pdSTL (probabilistic differentiable Signal Temporal Logic), a framework that unifies probabilistic semantics with differentiable robustness over belief trajectories. pdSTL employs interval-valued probabilistic semantics to compute conservative satisfaction bounds, propagated compositionally through the STL syntax tree. We formulate the temporal robustness evaluation as a recurrent, LSTM-style unfolding of STL operators, enabling linear-time, differentiable monitoring suitable for end-to-end trajectory optimization. We validate pdSTL on simulated obstacle avoidance, lane-change maneuvers, and real-world Crazyflie quadcopter flight experiments under aerodynamic disturbances. Results demonstrate that pdSTL achieves efficient optimization with formal probabilistic guarantees, significantly outperforming deterministic differentiable STL in maintaining safety margins under real-world uncertainty.
\end{abstract}

\section{Introduction}
Autonomous systems are increasingly deployed in safety-critical, uncertain environments where failures carry significant consequences. To operate reliably, these systems must not only accomplish complex, temporally extended tasks but also provide guarantees about safety and performance despite noise, disturbances, and unpredictable interactions. For example, a service robot in a hospital may need to visit patient rooms in a prescribed order, avoid congested corridors during peak hours, and periodically report to a central station, while contending with localization errors and dynamic human motion. Deterministic planning alone is insufficient in such settings. We therefore require formal specification frameworks that can both express rich temporal requirements and explicitly account for uncertainty, while remaining amenable to efficient computation and learning-based control.

Signal Temporal Logic (STL) has emerged as a powerful formal method for specifying and reasoning about spatiotemporal properties of continuous-time systems~\cite{donze2013signal}. STL extends Linear Temporal Logic with dense-time, real-valued signals, enabling precise specifications that describe mission objectives. Critically, STL admits quantitative semantics, called the robustness measure, which provides a continuous scalar score of how well a signal satisfies a specification~\cite{donze2010robust}. This quantitative view makes STL particularly attractive for robotics, where it has been successfully applied to motion planning, autonomous driving, and multi-agent coordination \cite{yuan2025continuous,silano2021power,barbosa2019guiding,puranic2021learning}.

However, real-world robotic and autonomous systems operate under significant uncertainty arising from sensor noise, actuation errors, model mismatch, and unpredictable interactions with dynamic environments. Traditional STL robustness is defined with respect to deterministic system signals and evaluates satisfaction based on nominal executions, thereby neglecting the stochastic nature of real-world system behavior. As a consequence, specifications that appear robust under deterministic semantics may be violated when deployed on physical systems, limiting their practical applicability.

\begin{figure}[t]
\centering
\includegraphics[width=\linewidth]{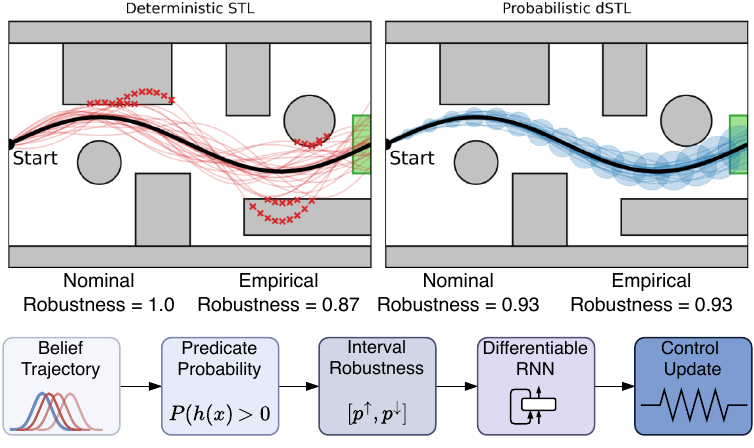}
\caption{The realization-gap problem in stochastic environments. (Top) Deterministic STL planning assumes a single nominal path, ignoring the "dispersion" caused by noise. This leads to a false sense of security (Nominal Robustness $= 1.0$) while the actual robot executions (Empirical Robustness $= 0.87$) collide with obstacles. (Bottom) Our proposed pdSTL framework solves this by optimizing directly over the \textit{belief trajectory}. By propagating probability intervals through a differentiable LSTM-style RNN roll out, pdSTL ensures that the planned safety margin ($0.93$) holds true during actual deployment ($0.93$), effectively closing the gap between theory and real-world performance.}
\label{fig:overview}
\end{figure}

To address this limitation, a growing body of work has extended STL to explicitly reason about uncertainty by incorporating probabilistic semantics. Early works introduced Probabilistic Signal Temporal Logic, interpreting predicate satisfaction as a random variable and defining satisfaction in terms of the probability that a specification holds, enabling controllers to maximize the likelihood of task completion under uncertainty~\cite{yoo2015control,kapoor2016probabilistic,tiger2020incremental}. Related approaches formulated STL synthesis as a probabilistic inference problem, allowing uncertainty to be propagated through the specification while retaining expressive temporal structure~\cite{lee2021signal}. In parallel, other approaches incorporated risk-aware \cite{lindemann2021reactive} and chance-constrained semantics \cite{jha2018safe}, enforcing STL specifications with bounded probabilities of violation through predictive or reactive control strategies~\cite{farahani2017shrinking}. More recently, a stochastic robustness interval was introduced, defining interval-valued probabilistic robustness semantics over continuous-time systems and taking the lower bound as a robustness measure~\cite{ilyes2022stochastic}. However, these methods are typically non-differentiable, preventing their direct use in modern gradient-based learning and optimization pipelines.

To make STL applicable in learning-based scenarios, researchers have explored differentiable formulations of STL robustness. A key development was the STLCG framework, which translated STL robustness formulas into computation graphs, enabling gradients to be backpropagated through temporal operators using automatic differentiation tools~\cite{leung2019backpropagation,leung2023backpropagation}. This formulation allowed STL robustness to be used directly as an optimization objective, enabling learning agents to optimize complex temporal behaviors without manual reward shaping. Subsequent work improved optimization behavior and scalability through smooth robustness approximations, efficient parallel evaluation, and tighter integration with deep learning frameworks, including recurrent architectures and neural-symbolic representations~\cite{meng2023signal,kapoor2025stlcg++,gilpin2020smooth,ma2020stlnet}. However, these differentiable approaches assume deterministic signals and overlook stochasticity.

In this paper, we bridge this gap by introducing probabilistic differentiable Signal Temporal Logic (pdSTL). Our central insight is that probabilistic STL robustness can be formulated as a differentiable recurrent computation—analogous to an LSTM-style unfolding of temporal operators over stochastic signals—thereby preserving probabilistic soundness while enabling end-to-end gradient-based optimization under uncertainty. This formulation makes STL directly applicable to modern learning-based cyber-physical systems operating in uncertain environments (Fig.~\ref{fig:overview}).

\noindent \textbf{Contributions.} Our contributions are: (i) an interval-valued probabilistic robustness semantics for STL over belief trajectories, using conservative probability bounds propagated through the STL syntax; (ii) casting STL evaluation as an LSTM-style RNN, updating lower and upper robustness bounds sequentially for linear-time bounded temporal evaluation while preserving interval consistency; and (iii) integrating this recurrent interval monitor into belief-space trajectory optimization, demonstrating effectiveness in robotic motion planning under temporal safety and reachability constraints with uncertainty.

\begin{figure*}[t]
    \centering
    \includegraphics[width=\linewidth]{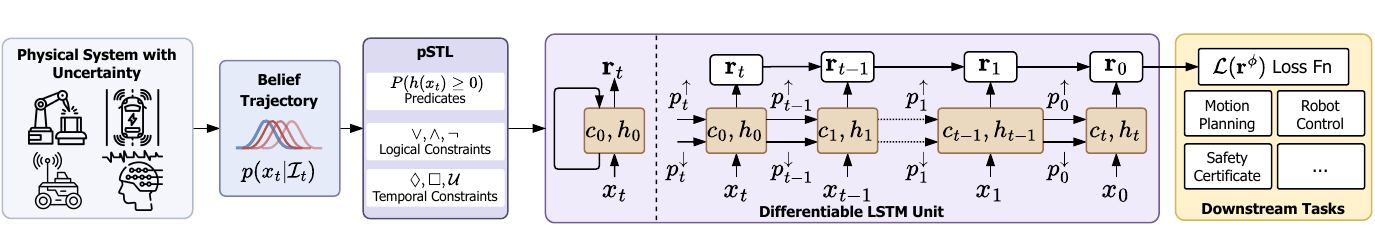}
    \caption{Differentiable probabilistic STL framework. For a given belief state $b$, is transformed into probabilistic STL predicate bounds $(p_t^{\uparrow}, p_t^{\downarrow})$. These bounds are evaluated through a differentiable temporal RNN that computes the conservative STL robustness which could then be used in gradient-based optimization.}
    \label{fig:pipeline}
\end{figure*}
\section{Approach}
In this section, we first formalize discrete-time belief-space dynamics, then define a stochastic robustness measure for temporal logic specifications. Crucially, we implement this evaluation as a recurrent, LSTM-style computation graph, enabling end-to-end differentiable trajectory optimization that propagates probability bounds through time efficiently. 

\subsection{Stochastic System Dynamics}
We consider a discrete-time stochastic dynamical system evolving over a finite time horizon $t = 0,1,\dots,T$. The system state $x_t \in \mathbb{R}^n$ evolves under control input $u_t \in \mathbb{R}^m$ as
\begin{equation}
x_{t+1} = f(x_t,u_t) + G(x_t,u_t)w_t, \qquad
y_t = h(x_t) + v_t,
\end{equation}
where $w_t$ and $v_t$ represent process and measurement noise, respectively, and $h$ may be nonlinear. The functions $f$ and $G$ define the (possibly nonlinear) state transition and noise scaling. Let $\mathcal{I}_t$ denote the information available up to time $t$, including the history of applied controls and observations.
Instead of reasoning over individual realizations of $x_t$, we operate in belief space. The belief at time $t$ is the posterior distribution
\begin{equation}
b_t(x) := p(x_t \mid \mathcal{I}_t).
\end{equation}
Given an initial belief $b_0$ and a candidate control sequence $\mathbf{u} = \{u_0, \dots, u_{T-1}\}$, the stochastic dynamics and observation model induce a predicted belief trajectory
\begin{equation}
\mathcal{B}(\mathbf{u}) = \{ b_0, b_1, \dots, b_T \},
\end{equation}
which captures the evolution of uncertainty over time. The proposed framework operates directly on $\mathcal{B}(\mathbf{u})$ to optimize control sequences under future uncertainty. No parametric assumptions are imposed on the beliefs $b_t$, which may be Gaussian, multimodal, heavy-tailed, or represented non-parametrically (e.g., via particles).

\subsection{Stochastic Robustness Measure}
\label{sec:storm}
Under the discrete-time stochastic dynamics above, each realization of the stochastic state process induced by $\mathcal{B}$ generates a trajectory that may or may not satisfy $\varphi$. Satisfaction therefore becomes a probabilistic event.
A standard approach to uncertainty handling is via chance constraints, which enforce that the probability of violating a constraint remains below a prescribed threshold at each time. However, such formulations do not naturally extend to STL specifications, whose temporal operators couple satisfaction across time intervals. For instance, \textit{eventually} ($\Diamond_{[a,b]}\varphi$) requires satisfaction at some time in $[a,b]$, whereas \textit{always} ($\Box_{[a,b]}\varphi$) requires satisfaction at all times in that interval. Enforcing independent per-time-step chance constraints fails to capture this structural asymmetry.
To address this limitation, we adopt the stochastic robustness interval introduced in \cite{ilyes2022stochastic} and formulate it directly over belief trajectories.
\paragraph{Stochastic Robustness Interval.}
Let $\varphi$ be an STL formula and $\mathcal{B} = {b_t}_{t=0}^T$ a belief trajectory. Following \cite{ilyes2022stochastic}, the stochastic robustness interval (SRI) is defined as
\begin{equation}
\mathbf{r}(\varphi,\mathcal{B}) =
\big[
p^\downarrow(\varphi,\mathcal{B}),
p^\uparrow(\varphi,\mathcal{B})
\big]
\subseteq [0,1],
\end{equation}
with
\begin{equation}
0 \le
p^\downarrow(\varphi,\mathcal{B})
\le
p^\uparrow(\varphi,\mathcal{B})
\le 1.
\end{equation}
The lower bound $p^\downarrow(\varphi,\mathcal{B})$ provides a conservative guarantee on the probability that realizations of the stochastic process satisfy $\varphi$, while $p^\uparrow(\varphi,\mathcal{B})$ gives an upper bound. The stochastic robustness measure (SRM) is then
\begin{equation}
\mathrm{SRM}(\varphi,\mathcal{B}) := \min \mathbf{r}(\varphi,\mathcal{B}).
\end{equation}
By design, the lower bound ($p^\downarrow$) represents a cautious estimate, while the upper bound ($p^\uparrow$) is an optimistic one. Our system is built so that these two never "cross over." If the cautious estimate starts lower than the optimistic one, it will stay that way through every step of the calculation.

Intuitively, the stochastic robustness interval propagates probability bounds along the logical and temporal structure of the STL formula, combining atomic predicates using probability inequalities and temporal operators over their corresponding time windows. If $\mathrm{SRM}(\varphi,\mathcal{B}) = 1$, the formula is satisfied almost surely; intermediate values quantify the guaranteed probability of satisfaction. Unlike deterministic robustness $\mathbf{r}(x,\varphi,t) \in \mathbb{R}$, which measures geometric distance from violation, stochastic robustness captures confidence under uncertainty.
\paragraph{Relation to Prior Work and Our Contribution.}
The stochastic robustness interval follows the StoRM framework of \cite{ilyes2022stochastic}. Our contribution extends this foundation by integrating stochastic robustness into a differentiable optimization framework, enabling gradient-based planning and learning directly in belief space.

\subsection{Differentiable Signal Temporal Logic}
We develop a differentiable framework for evaluating STL specifications over belief trajectories, enabling gradient-based trajectory optimization under probabilistic constraints. The probabilistic semantics in Section~\ref{sec:storm} define robustness intervals via recursive min/max operations, which are conceptually clear but non-differentiable. We extend the STLCG framework \cite{leung2023backpropagation} to operate on probabilistic intervals instead of scalar traces, replacing min/max with smooth approximations and propagating bounds through a recurrent computation graph, achieving linear complexity and end-to-end differentiability (Fig.~\ref{fig:pipeline}). For any formula $\varphi$ and time $t$, define the stochastic robustness interval as
\begin{equation}
\mathbf{r}_t^\varphi
:= [p_t^\downarrow(\varphi),\ p_t^\uparrow(\varphi)]
\in [0,1]^2,
\end{equation}
where $p_t^\downarrow(\varphi)$ and $p_t^\uparrow(\varphi)$ denote lower and upper probability bounds evaluated on the suffix of the predicted belief trajectory $\mathcal{B}_t$. For each atomic predicate of the form $\mu(x_t) \ge 0$, we compute the probability of satisfaction under the belief $b_t(x) = p(x_t \mid \mathcal{I}_t)$ as
\begin{equation}
p_t = \mathbb{P}_{x \sim b_t}(\mu(x) \ge 0),
\end{equation}
which corresponds to the probability mass of the feasible region under the belief distribution.

For Gaussian beliefs, $b_t = \mathcal{N}(\mu_t, \Sigma_t)$, and linear predicates $a^\top x \le c$, this probability is available in closed form:
\begin{equation}
\mathbb{P}(a^\top x \le c)
=
\Phi\!\left(
\frac{c - a^\top \mu_t}{\sqrt{a^\top \Sigma_t a}}
\right).
\end{equation}

For rectangular goal or obstacle regions, predicates are decomposed into conjunctions of half-space constraints, and Fr\'echet bounds are used to conservatively combine probabilities across dimensions without assuming independence. These probabilities serve as inputs to the stochastic robustness interval computation described above.

All operations are differentiable with respect to $(\mu_t, \Sigma_t)$, enabling gradients to propagate from the probabilistic STL objective through belief dynamics to control inputs.

\subsubsection{\bf{Non-Temporal Operators}}
Non-temporal operators (negation, conjunction, disjunction) act element-wise operation across the signal and do not require temporal aggregation. Negation swaps and complements the lower and upper bounds of the interval:
\begin{equation}
\mathbf{r}_t^{\neg\varphi}
= [1 - p_t^\uparrow(\varphi),\ 1 - p_t^\downarrow(\varphi)]
\end{equation}
Conjunction and disjunction are computed using Fr\'echet bounds to conservatively combine probabilities without independence assumptions:
\begin{equation}
\begin{aligned}
\mathbf{r}_t^{\varphi_1 \wedge \varphi_2}
&=
\Big[
\max\!\big(0,\,
p_t^\downarrow(\varphi_1)
+
p_t^\downarrow(\varphi_2)
-
1\big),
\\[-2pt]
&\qquad
\min\!\big(
p_t^\uparrow(\varphi_1),
p_t^\uparrow(\varphi_2)
\big)
\Big],
\\[6pt]
\mathbf{r}_t^{\varphi_1 \vee \varphi_2}
&=
\Big[
\max\!\big(
p_t^\downarrow(\varphi_1),
p_t^\downarrow(\varphi_2)
\big),
\\[-2pt]
&\qquad
\min\!\big(
1,\,
p_t^\uparrow(\varphi_1)
+
p_t^\uparrow(\varphi_2)
\big)
\Big].
\end{aligned}
\end{equation}
To enable smooth differentiable computation, the min and max operations are replaced with the log-sum-exponential  approximation~\cite{leung2023backpropagation}. As the smoothing parameter $\beta \to \infty$, the relaxed operators converge to the exact SRI bounds.

\subsubsection{\bf{Temporal Operators}}
Temporal operators aggregate probability intervals over a bounded window $[a,b]$. To maintain $O(T)$ complexity while handling the future-looking nature of STL, we reformulate the evaluation as a \textit{backward recurrent computation}. By leveraging the RNN-style computation graph architecture proposed in \cite{leung2023backpropagation}, we compute robustness values in a single pass over the trajectory. Starting from $t=T$ and moving to $t=0$, we maintain a shift-register $z_t^\varphi$ representing the required future window:
\begin{equation}
z_t^\varphi =
\begin{bmatrix}
\mathbf{r}_{t+a}^\varphi \\
\mathbf{r}_{t+a+\Delta t}^\varphi \\
\vdots \\
\mathbf{r}_{t+b}^\varphi
\end{bmatrix}
=
\begin{bmatrix}
c_{t+a}^\varphi & h_{t+a}^\varphi \\
\vdots & \vdots \\
c_{t+b}^\varphi & h_{t+b}^\varphi
\end{bmatrix}
\end{equation}

where $W = \lfloor (b-a)/\Delta t \rfloor + 1$ is the window size. As $t$ decreases, the state is updated via:
\begin{equation}
z_{t-\Delta t}^\varphi = M z_t^\varphi + B \mathbf{r}_{t+a-\Delta t}^\varphi
\end{equation}
with $B = [1, 0, \dots, 0]^\top$ and $M$ being the standard down-shift matrix. This sliding window allows the robustness at time $t$ to be computed via a smooth reduction over the current register $z_t^\varphi$ in constant time. The construction for bounded intervals $[a,b]$ extends naturally to unbounded cases by modifying the window extraction logic within the recurrent state.

\paragraph{Always and Eventually Operators}  
The output of bounded temporal operators is obtained by applying a smooth reduction across the hidden state:
\begin{equation}
\begin{aligned}
\mathbf{r}_t^{\Box_{[a,b]}\varphi}
&=
\begin{bmatrix}
\widetilde{\min}_\beta\{p_s^\downarrow(\varphi)\}_{s\in[t+a,t+b]} \\
\widetilde{\min}_\beta\{p_s^\uparrow(\varphi)\}_{s\in[t+a,t+b]}
\end{bmatrix} \\
\mathbf{r}_t^{\Diamond_{[a,b]}\varphi}
&=
\begin{bmatrix}
\widetilde{\max}_\beta\{p_s^\downarrow(\varphi)\}_{s\in[t+a,t+b]} \\
\widetilde{\max}_\beta\{p_s^\uparrow(\varphi)\}_{s\in[t+a,t+b]}
\end{bmatrix}
\end{aligned}
\end{equation}
This recurrence preserves linear complexity while maintaining differentiability via smooth min/max relaxations.

\paragraph{Until Operator}
The bounded until operator $\varphi_1 \,\mathcal{U}_{[a,b]}\,\varphi_2$ requires that $\varphi_2$ holds at some future time $k \in [t+a, t+b]$, and $\varphi_1$ holds continuously from the current time $t$ until $k$. We implement this as a coupled recurrence:
\begin{equation}
\begin{split}
p_t^\downarrow(\varphi_1 \mathcal{U}_{[a,b]} \varphi_2) = \widetilde{\max}_{k \in [t+a, t+b]} \Big( \widetilde{\min} \big( p_k^\downarrow(\varphi_2), \\
\widetilde{\min}_{s \in [t, k]} p_s^\downarrow(\varphi_1) \big) \Big)
\end{split}
\end{equation}
An analogous computation is performed for the upper bound $p_t^\uparrow$. By re-using the shifted suffix values in the recurrence, we preserve linear-time evaluation.

\subsection{Illustrative Example}
To concretely demonstrate the sliding-window computation, we trace the evaluation of $\Box_{[1,2]}(x \ge 50)$ on a discrete piecewise signal (Fig.~\ref{fig:piecewise_example}). Let the predicate probability intervals be $\mathbf{r}_t = [p_t^\downarrow, p_t^\uparrow]$. The \emph{Always} operator with window $[1, 2]$ dictates a sliding window of size $W = 2$. At each step $t$, the recurrent shift register extracts the future window $\{t+1, t+2\}$ and applies the $\widetilde{\min}$ reduction to output the aggregated interval bounds, as detailed in Table~\ref{tab:predicate_window}.

\begin{figure}[t]
\centering
\includegraphics[width=\linewidth]{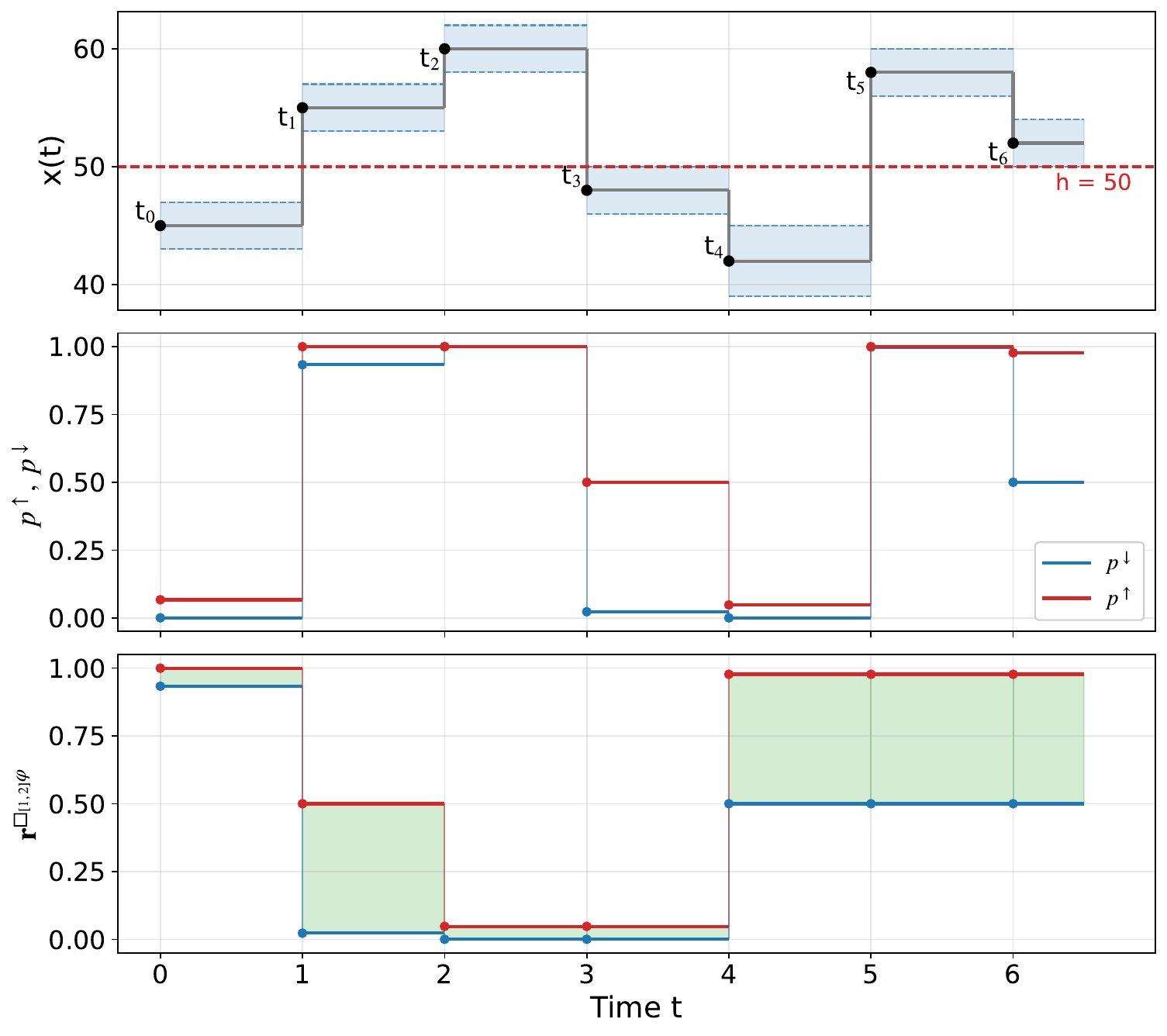}
\caption{$\Box_{[1,2]}(x \ge 50)$ valuation of the probabilistic STL monitor.
Top: Piecewise signal with uncertainty bounds and threshold.
Middle: Predicate satisfaction probability bounds.
Bottom: Temporal operator output computed via recurrent aggregation.}
\label{fig:piecewise_example}
\end{figure}

\begin{table}[htpb]
\centering
\caption{Predicate Probability Intervals and Sliding-Window Evaluation for $\Box_{[1,2]}\varphi$}
\label{tab:predicate_window}
\resizebox{\columnwidth}{!}{%
\addtolength{\tabcolsep}{-2pt}
\begin{tabular}{c c c c c}
\toprule
$t$ & $p_t^\downarrow$ & $p_t^\uparrow$ & Suffix $s \in [t+1, t+2]$ & $\mathbf{r}_t^{\Box_{[1,2]}\varphi}$ \\
\midrule
0 & 0.0002 & 0.0668 & $\{1, 2\}$ & [0.9332, 0.9998] \\
1 & 0.9332 & 0.9998 & $\{2, 3\}$ & [0.0228, 0.5000] \\
2 & 1.0000 & 1.0000 & $\{3, 4\}$ & [0.0001, 0.0478] \\
3 & 0.0228 & 0.5000 & $\{4, 5\}$ & [0.0001, 0.0478] \\
4 & 0.0001 & 0.0478 & $\{5, 6\}$ & [0.5000, 0.9772] \\
5 & 0.9987 & 1.0000 & $\{6, 7\}$ & --- \\
6 & 0.5000 & 0.9772 & ---        & --- \\
\bottomrule
\end{tabular}%
}
\end{table}

\section{Results and Discussion}
We now demonstrate the efficacy of the pdSTL framework in three cases. First,  we test the framework with obstacle avoidance then test on a lane change maneuver. Finally, we present a real-world validation experiment.

\subsection{2D Motion Planning}
We consider a 2D motion planning scenario where an autonomous agent must navigate a cluttered environment under stochastic uncertainty (Fig.~\ref{fig:2d_motion_planning}a)). The agent is modeled as a discrete-time stochastic single integrator:
$x_{t+1} = x_t + u_t \Delta t + w_t, \quad w_t \sim \mathcal{N}(0,Q),\ u_t \in [-u_{\max}, u_{\max}]^2$,
with state $x_t \in \mathbb{R}^2$ and control $u_t$. The workspace contains a goal region, an intermediate visit region $V$, and static obstacles $\mathcal{O} = {O_1,\dots,O_N}$. The mission is formalized as the STL specification:

\begin{equation}
\varphi = \Diamond_{[0,T]} (x \in G) \wedge \Diamond_{[0,T]} (x \in V) \wedge \Box_{[0,T]} \bigwedge_{O_i \in \mathcal{O}} (x \notin O_i).
\end{equation}
Trajectory synthesis is cast as unconstrained optimization over the control sequence $\mathbf{u}_{0:T-1}$ by minimizing:

\begin{equation}
\mathcal{J}(\mathbf{u}_{0:T-1}) = -\log\big(\text{SRM}(\varphi,\mathcal{B}) + \varepsilon\big) + \lambda_u \sum{t}|u_t|^2,
\end{equation}

Given an initial belief $b(0)$, horizon $T$, and learning rate $\eta$, the control sequence is refined iteratively by backpropagating gradients through the differentiable pdSTL graph.

\begin{figure}[th!]
\centering
\includegraphics[width=\linewidth]{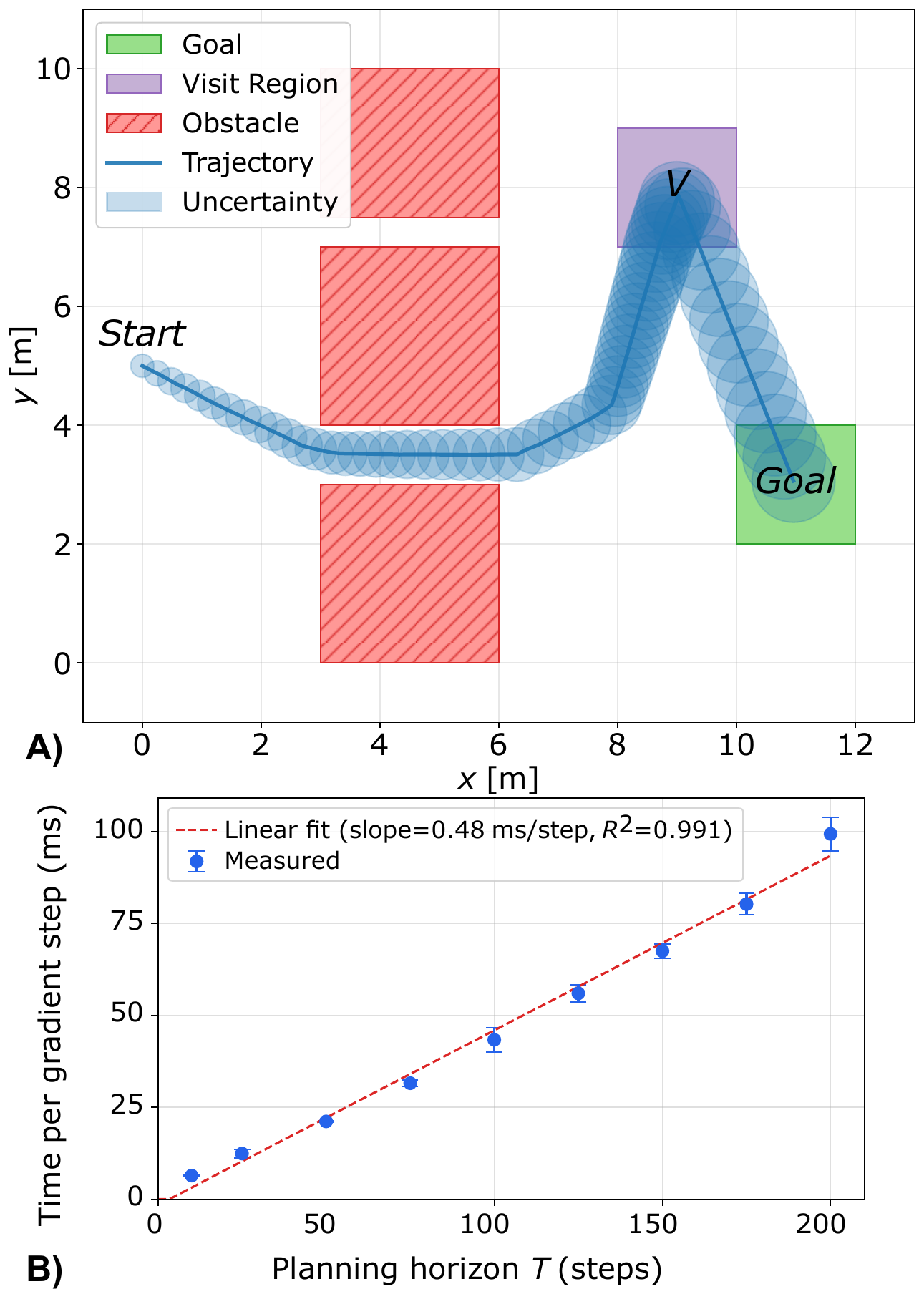}
\caption{A) Belief-space trajectory optimization under pdSTL. Blue ellipses indicate propagated uncertainty. The differentiable monitor maintains a robust satisfaction margin, guiding the trajectory away from obstacles while satisfying temporal objectives. B) Gradient computational time vs. planning horizon varying linearly.}
\label{fig:2d_motion_planning}
\end{figure}
We compare pdSTL against a deterministic STLCG baseline. Both planners produce fixed open-loop sequences, which are tested under $N=1000$ Monte Carlo rollouts with a process noise ($\sigma_q = 0.03$). While nominal distances to obstacles are similar, the deterministic plan is brittle (Table~\ref{tab:single_shot_comparison}): its probabilistic robustness drops to $0.834$, yielding 87–88\% success and safety rates. In contrast, pdSTL explicitly optimizes the probability of satisfaction over the belief trajectory, achieving a planning-time guarantee of $0.930$ that aligns with empirical performance (success 97.8\%, safety 98.4\%), maintaining larger obstacle clearance. Also, as shown in Fig.~\ref{fig:2d_motion_planning}b), the computation time per gradient step increases linearly with the planning horizon $T$.

\begin{table}[t]
\centering
\caption{Comparison of deterministic (STLCG) vs probabilistic (pdSTL) planning ($N=1000$ Monte Carlo trials). }
\label{tab:single_shot_comparison}
\begin{tabular}{lcc}
\toprule
Metric & STLCG & pdSTL \\
\midrule
$\eta$ (nominal distance) & 0.200 & 0.199 \\
SRM (planning) & - & 0.930 \\
SRM deployed & 0.834 & 0.930 \\
MC success rate & 87.2\% & 97.8\% \\
MC safety rate & 87.7\% & 98.4\% \\
Min obstacle clearance (m) & $0.153 \pm 0.122$ & $0.259 \pm 0.097$ \\
\bottomrule
\end{tabular}
\end{table}

\subsection{Lane-Change for Autonomous Vehicle}

\begin{figure}[ht!]
\centering
\includegraphics[width=\linewidth]{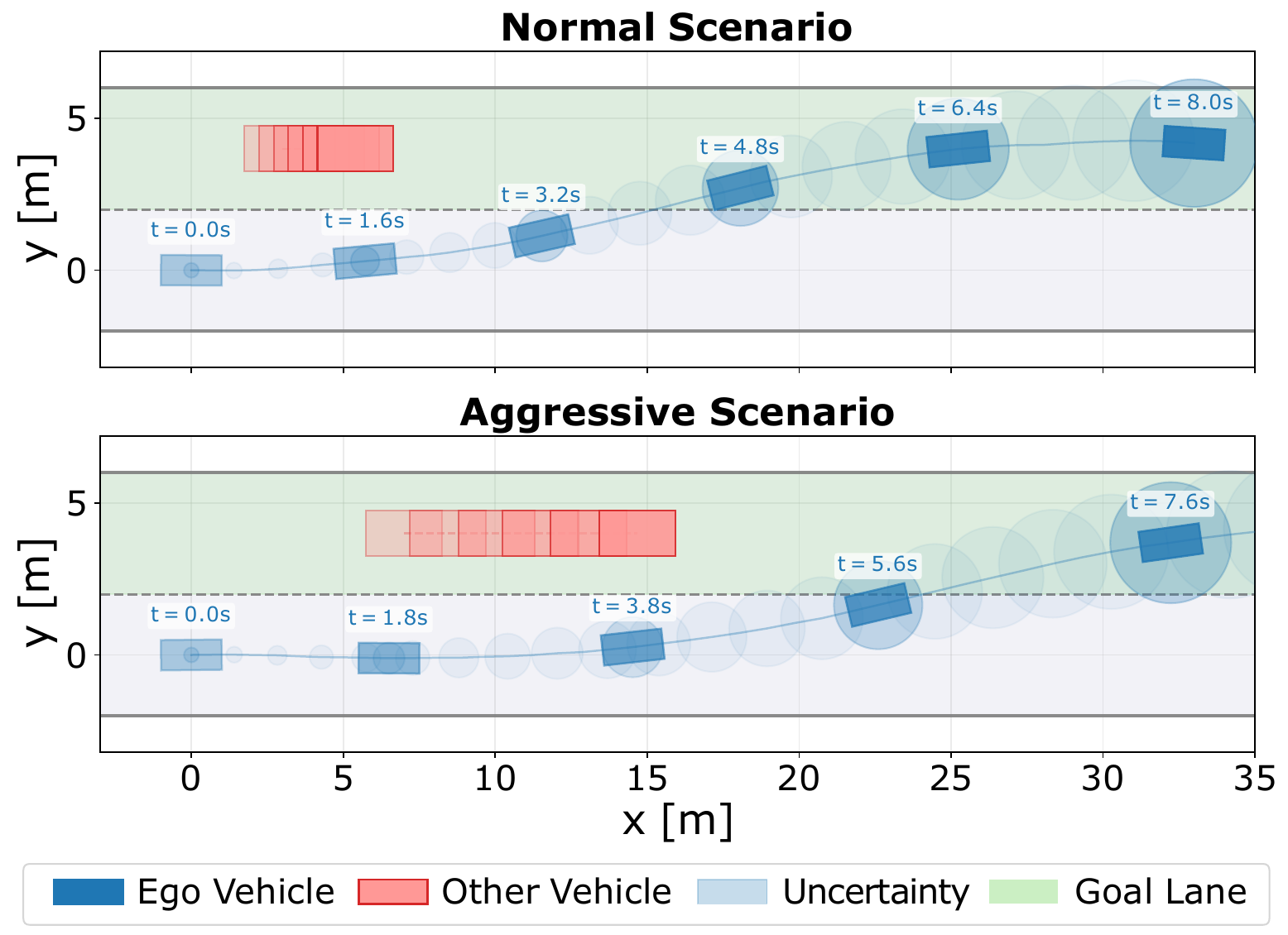}
\caption{Lane-merging behavior across two scenarios. Top: normal scenario where the obstacle moves at $0.3$\,m/s. Bottom: aggressive scenario where the obstacle moves at $0.8$\,m/s, reducing the merge opportunity.}
\label{fig:lane_change}
\end{figure}

We consider a lane-merging scenario where an ego vehicle must safely merge into a target lane while avoiding a moving obstacle. The ego vehicle follows stochastic double-integrator dynamics:
\begin{equation}
\begin{aligned}
x_t &= [p_x, p_y, v_x, v_y]^\top, \\
u_t &= [a_x, a_y]^\top, \quad \|u_t\|_\infty \le u_{\max} = 2.5~\text{m/s}^2.
\end{aligned}
\end{equation}
with timestep $\Delta t = 0.2$ s and additive Gaussian position noise ($\sigma_q = 0.01$ m). The obstacle moves at constant velocity ($0.3$ or $0.8$ m/s depending on the scenario) along the target lane from $p_{\text{obs},0} = [3.0, 4.0]^\top$ m, and a collision occurs if $|p - p_{\text{obs}}|_2 < 2.25$ m.
The ego vehicle uses pdSTL in a Model Predictive Control fashion: at each step, a finite-horizon problem ($H=40$ steps) is solved, the first control applied, and the horizon shifted forward, up to $T_{\text{sim}} = 100$ steps. 
Two representative scenarios are evaluated. In the \emph{normal scenario}, the obstacle starts at $p_{\text{obs},0}=[3.0,4.0]^\top$\,m.
In the \emph{aggressive scenario}, the obstacle starts further ahead at $p_{\text{obs},0}=[7.0,4.0]^\top$\,m and moves faster, reducing the available merge opportunity. The lane-change specification is given by:
\begin{equation}
\varphi_{\text{merge}} =
\Diamond_{[0,H]}(|p_y - y_{\text{target}}| \le 2) \wedge
\Box_{[0,H]}(|p - p_{\text{obs}}|_2 \ge 2.25),
\end{equation}
requiring eventual alignment with the target lane while maintaining a probabilistic safety margin.

Figure~\ref{fig:lane_change} illustrates the closed-loop execution of the lane-change maneuver under the two scenarios. In the normal scenario, the obstacle moves slowly along the target lane, allowing the ego vehicle to merge once sufficient separation is achieved. In contrast, the aggressive scenario reduces the available merge window by introducing a faster-moving obstacle, requiring the planner to reason more carefully about the timing of the maneuver. In both cases, the pdSTL controller delays the merge while the predicted belief trajectory overlaps the obstacle’s safety region, initiating the lateral transition only when the probabilistic safety margin becomes feasible. This behavior results in trajectories that naturally adapt to the available space and timing of the maneuver. While the aggressive scenario operates with reduced spatial margins, the optimizer still converges to a solution satisfying $\rho_{\downarrow}(\varphi) \ge 0.90$, certifying that the high-speed maneuver maintains the same probabilistic safety guarantee as the more conservative scenario. For a planning horizon of $H=40$, a complete forward–backward pass of the pdSTL monitor requires $0.49$\,ms on an Intel i9 CPU.

\subsection{Real-World Obstacle Avoidance}
\begin{figure}[t]
    \centering
    \includegraphics[width=\linewidth]{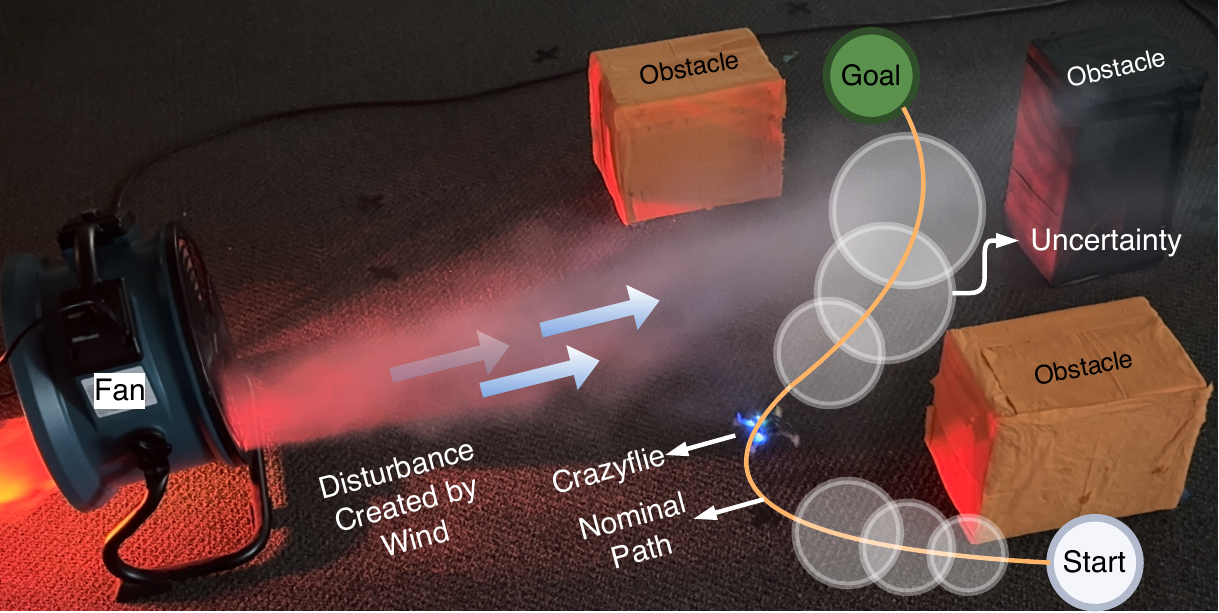}
    \caption{Experimental setup for real-world validation of pdSTL. A Crazyflie quadrotor navigates from a Start position to a defined Goal while avoiding multiple Obstacles. To evaluate the system’s robustness under stochastic dynamics, a fan is used to introduce a disturbance  across the flight area.}
    \label{fig:real-world-example}
\end{figure}
We validate the pdSTL framework using a Crazyflie quadrotor performing an obstacle avoidance task under real aerodynamic disturbances. Experiments were conducted in a $2\,\text{m} \times 2\,\text{m}$ netted arena using Lighthouse-based localization and ROS2 Humble (Fig.~\ref{fig:real-world-example}). By maintaining a constant altitude, the task is reduced to planar motion where the state $\hat{p}_t = [\hat{x}_t,\hat{y}_t]^\top$ is estimated via an onboard EKF. To bridge the gap between the physical system and our probabilistic semantics, we modeled the vehicle belief $b_t$ as a Gaussian distribution $\mathcal{N}(\hat{p}_t, \Sigma_t)$. While the mean $\hat{p}_t$ follows the nominal planned state, the covariance $\Sigma_t$ was pre-characterized by recording tracking error residuals under various fan intensities during preliminary flight trials. This empirical noise model allows the pdSTL optimizer to reason over the predicted dispersion of the quadrotor trajectories during the offline optimization phase. The environment consists of three rectangular obstacles $O_i$ and a goal region $G$, with the mission specified as:
\begin{equation}
\varphi = \Diamond_{[0,T]} (\hat{p}_t \in G) \;\wedge\; \Box_{[0,T]} \bigwedge_{i=1}^{3} (\hat{p}_t \notin O_i).
\end{equation}

We conducted a comparative study between two planning strategies: (i) Deterministic Optimization, a baseline assuming noise-free dynamics, and (ii) pdSTL Optimization (Ours), which maximizes the conservative lower-bound satisfaction probability with a target threshold of $\mathbf{r}_{\downarrow} \ge 0.95$. Trajectories were optimized offline and executed on hardware using identical waypoint tracking controllers and gains to ensure the only variable was the robustness formulation. To introduce stochasticity, a calibrated fan was placed at the arena boundary ($\approx 1.5$~m from the obstacles) to generate persistent lateral airflow. We evaluated three intensity levels (Settings 1-3), where Setting 1 corresponds to $\approx 1.35$~m/s, and Settings 2 and 3 reach $\approx 2$~m/s and $3.5$~m/s, respectively. Prior to optimization, we initialized the belief state covariance based on empirical disturbance measurements collected for each fan setting, using $\Sigma_0 = 0.001\mathrm{m}^2$, $0.006\mathrm{m}^2$, and $0.02\mathrm{m}^2$ for Settings 1–3, respectively. Performance is summarized in Table~\ref{tab:results} based on success rate (\%), mean obstacle clearance (cm), and total safety violations.
\begin{table}[t]
  \centering
  \caption{Safety rate (\%), minimum obstacle clearance (cm), and violation count across 20 runs per fan setting.
  Safety ($\%$): $= (N-\text{violations})/N \times 100\,\%$.
  Clearance: mean across runs (violated/crashed runs assigned 0\,cm.)}
  \label{tab:results}
  \setlength{\tabcolsep}{4pt}
  \begin{tabular}{c ccc ccc}
    \toprule
      & \multicolumn{3}{c}{Deterministic}
      & \multicolumn{3}{c}{pdSTL (ours)} \\
    \cmidrule(lr){2-4}\cmidrule(lr){5-7}
    Fan Level
      & \shortstack{Success\\(\%)} 
      & \shortstack{Clearance\\(cm, $\uparrow$)} 
      & Viols.
      & \shortstack{Success\\(\%)} 
      & \shortstack{Clearance\\(cm, $\uparrow$)} 
      & Viols. \\
    \midrule
    1        & $95.0$  & $3.4 \pm 2.0$ & 1
             & \textbf{$100.0$} & \textbf{$7.9 \pm 1.9$} & \textbf{0} \\
    2        & $95.0$  & $4.4 \pm 2.2$ & 1
             & \textbf{$100.0$} & \textbf{$8.8 \pm 1.8$} & \textbf{0} \\
    3        & $70.0$  & $4.6 \pm 4.1$ & 6
             & \textbf{$85.0$}  & \textbf{$4.5\pm 4.4$} & \textbf{3} \\
    \bottomrule
  \end{tabular}
\end{table}

As disturbance increases, planner performance diverges. Under mild airflow (Fan Levels~1--2), both planners achieve high success rates, but pdSTL consistently maintains larger safety margins. At Fan Level~1, the deterministic planner achieves $95\%$ success with $3.4 \pm 2.0$\,cm mean clearance, whereas pdSTL achieves $100\%$ success with $7.9 \pm 1.9$\,cm. A similar trend appears at Fan Level~2, where both planners succeed frequently but pdSTL maintains nearly double the clearance. 

At higher disturbance (Fan Level~3), the deterministic planner degrades more rapidly, reaching $70\%$ success with six safety violations, while pdSTL maintains $85\%$ success with only three violations. Although clearance decreases for both planners, the probabilistic approach safely completes more trials, demonstrating that incorporating probabilistic STL semantics improves robustness by explicitly reasoning over belief uncertainty.

\section{Limitations}
While pdSTL provides a differentiable framework for probabilistic reasoning over belief trajectories, several avenues for refinement remain. First, the current formulation does not explicitly regulate covariance dynamics; the optimizer maximizes the lower bound of satisfaction probability based on predicted uncertainty but does not incorporate active covariance-shaping or information-seeking objectives. Consequently, uncertainty is mitigated through path selection rather than being actively reduced via control. 

Furthermore, the pdSTL's performance is naturally coupled to the accuracy of the noise model $\Sigma_t$. Currently estimated offline, $\Sigma_t$ may lead to over-optimistic safety thresholds if abrupt environmental changes exceed pre-characterized residuals. Integrating online adaptive noise estimation or dual-control to update $\Sigma_t$ in real-time is a vital extension to maintain probabilistic guarantees under model mismatch.

Second, while Fr\'echet bounds ensure compositional soundness without assuming predicate independence, they can introduce excessive conservativeness in the presence of strong state or temporal dependencies. Future work will explore tighter analytical bounds or differentiable sampling-based approximations to reduce this pessimism. 

Finally, although our recurrent formulation is linear-time, scalability is still constrained by the maximum temporal window size $W$. For very long-horizon missions with dense constraints, memory overhead remains a challenge; this could be mitigated through hierarchical abstractions or more aggressive receding-horizon schemes.
\section{Conclusions}
In this work, we introduced probabilistic differentiable Signal Temporal Logic (pdSTL), a framework that bridges formal temporal reasoning with gradient-based optimization in belief space. By defining interval-valued probabilistic semantics and implementing them as a recurrent, sliding-window computation, we enabled linear-time, differentiable monitoring of complex specifications under uncertainty. 

Our validation across simulated motion planning, lane-merging maneuvers, and real-world quadrotor experiments demonstrates that pdSTL effectively helps bridge the "realization gap" inherent in deterministic planners. Unlike traditional differentiable STL approaches, pdSTL produces trajectories where empirical success rates align closely with theoretical satisfaction guarantees. While the current framework focuses on optimizing paths within a predicted belief evolution, future research will explore active covariance-shaping and dual-control strategies to proactively reduce uncertainty.

\section*{Appendix}
\label{appendix:theory}
We state formal properties of the pdSTL monitor and provide brief proof sketches.

\begin{proposition}[Interval validity]
\label{prop:interval_validity_app}
Assume every atomic predicate $\mu$ yields bounds
$0 \le p_t^\downarrow(\mu) \le p_t^\uparrow(\mu) \le 1$ for all $t$.
Then for any formula $\varphi$ constructed using
$\neg,\wedge,\vee,\Box_{[a,b]},\Diamond_{[a,b]}$,
pdSTL produces bounds satisfying
\[
0 \le p_t^\downarrow(\varphi) \le p_t^\uparrow(\varphi) \le 1
\quad \forall t.
\]
The temporal cases follow because bounded temporal operators apply min/max reductions to collections of values already lying in $[0,1]$, thereby preserving interval validity.
\end{proposition}

\begin{proof}[Proof Sketch]
We proceed by structural induction on $\varphi$.
For negation,
$[p^\downarrow,p^\uparrow] \mapsto [1-p^\uparrow,\,1-p^\downarrow]$,
which preserves the unit interval and ordering.
For conjunction, the Fr\'echet lower bound
$\max(0, p_1^\downarrow + p_2^\downarrow - 1)$ and upper bound
$\min(p_1^\uparrow, p_2^\uparrow)$ both lie in $[0,1]$. Moreover, since
$p_i^\downarrow \le p_i^\uparrow$,
\[
\max(0, p_1^\downarrow + p_2^\downarrow - 1)
\le
\min(p_1^\uparrow, p_2^\uparrow),
\]
so ordering is preserved. The disjunction case is analogous.
By definition, the probability bounds are restricted to the unit interval $[0,1]$. Numerical values outside this range due to smooth approximations are projected back onto  $[0,1]$ via a saturation operator.
\end{proof}

\begin{proposition}[Monotonicity of smooth operators]
\label{prop:monotonicity_app}
Let $\widetilde{\max}_\beta$ denote the log-sum-exp relaxation and define
$\widetilde{\min}_\beta(x) := -\widetilde{\max}_\beta(-x)$.
If $x_i \le y_i$ for all $i$, then
\[
\widetilde{\max}_\beta(x_1,\dots,x_n)
\le
\widetilde{\max}_\beta(y_1,\dots,y_n),
\]
and similarly for $\widetilde{\min}_\beta$.
\end{proposition}

\begin{proof}[Proof Sketch]
For each $i$, $\exp(\beta x_i) \le \exp(\beta y_i)$, hence
$\sum_i \exp(\beta x_i) \le \sum_i \exp(\beta y_i)$. Taking $\log(\cdot)$
preserves inequality. The min result follows by duality.
\end{proof}

\begin{theorem}[Consistency with exact SRI as $\beta\to\infty$]
\label{thm:beta_consistency_app}
Let $\widetilde{\max}_\beta$ be defined by
$\widetilde{\max}_\beta(x_1,\dots,x_n)=\frac{1}{\beta}\log\sum_i e^{\beta x_i}$,
and $\widetilde{\min}_\beta(x) := -\widetilde{\max}_\beta(-x)$.
For any finite input set,
\[
\widetilde{\max}_\beta(x_1,\dots,x_n)\to \max_i x_i,
\qquad
\widetilde{\min}_\beta(x_1,\dots,x_n)\to \min_i x_i
\]
as $\beta\to\infty$.
Consequently, the relaxed pdSTL bounds converge pointwise to the exact SRI
bounds for $\wedge,\vee,\Box,\Diamond$ as $\beta\to\infty$.
\end{theorem}

\begin{proof}[Proof Sketch]
Log-sum-exp satisfies the standard bound
\[
\max_i x_i
\le
\widetilde{\max}_\beta(x)
\le
\max_i x_i + \frac{\log n}{\beta},
\]
so $\widetilde{\max}_\beta(x)\to \max_i x_i$ as $\beta\to\infty$.
The min case follows by duality. Since pdSTL composes finitely many such
operators over finite windows, convergence holds at each node of the formula
tree.
\end{proof}

\begin{proposition}[Linear-time bounded monitoring]
\label{prop:linear_time_app}
For a fixed STL formula $\varphi$ with bounded temporal intervals and horizon $T$,
pdSTL evaluation requires $O(T|\varphi|)$ time and $O(W|\varphi|)$ memory, where
\[
W = \max_{[a,b]\in\varphi}\Big(\big\lfloor (b-a)/\Delta t \big\rfloor + 1\Big)
\]
is the maximum window size over all bounded temporal operators in $\varphi$.
\end{proposition}

\begin{proof}[Proof Sketch]
Each syntax tree node is evaluated once per time step.
Bounded temporal operators reuse a window state (shift register) updated in
constant time per step. Thus time is $O(T)$ per node, giving $O(T|\varphi|)$
overall. Memory stores window states per node, yielding $O(W|\varphi|)$.
\end{proof}

\begin{theorem}[Differentiability of pdSTL]
\label{thm:differentiability_app}
Assume the belief dynamics
$b_{t+1} = \mathrm{update}(b_t, u_t)$ are differentiable with respect to $u_t$,
and the predicate satisfaction probability $p_t(\mu)$ is continuously
differentiable with respect to the belief parameters. Then $\mathrm{SRM}(\varphi, \mathcal{B})$ is almost everywhere smoothly differentiable with respect to the 
control sequence for any fixed $\beta>0$.
\end{theorem}

\begin{proof}[Proof Sketch]
By assumption, Jacobians of the belief update and predicate probabilities exist.
The pdSTL monitor composes affine transformations with smooth log-sum-exp
relaxations, which are smooth for $\beta>0$. Therefore the mapping
$\mathbf{u}\mapsto \mathrm{SRM}(\varphi,\mathcal{B}(\mathbf{u}))$ is
differentiable except at measure-zero points corresponding to coincident smooth
extrema in the relaxed reductions. Hence the entire monitor defines a
differentiable computation graph supporting backpropagation.
\end{proof}

\printbibliography

\end{document}